\documentclass{statsoc}
\usepackage[a4paper]{geometry}
\usepackage[T1]{fontenc}
\usepackage[utf8]{inputenc}

\usepackage{listings}
\usepackage{clrscode}
\usepackage[table]{xcolor}
\usepackage{array}
\usepackage{multirow}
\usepackage{tikz}
\usepackage{stmaryrd}
\usepackage{graphicx}
\usepackage{amsmath,amssymb}
\usepackage{stfloats}
\usepackage{float}
\usepackage{comment}

\newtheorem{theorem}{Theorem}[section]

\newtheorem{lemma}[theorem]{Lemma}

\newcommand{\sv}{\textsc{sv}}

\usepackage[utf8]{inputenc}
\usepackage[T1]{fontenc}

\newcommand{\changed}[1]{
  #1
}

\newcommand{\RR}{\mathbb R}

\DeclareMathOperator*{\argmin}{arg\,min}
\DeclareMathOperator*{\Diag}{Diag}

\DeclareMathOperator*{\maximize}{maximize}
\DeclareMathOperator*{\minimize}{minimize}

\newfloat{Algorithm}{thp}{lop}
\floatname{Algorithm}{Algorithm}

\usepackage{natbib}

\usepackage{algorithm}
\usepackage{algorithmic}

\definecolor{lightgray}{rgb}{0.9,0.9,0.9}
\definecolor{pastelblue}{RGB}{213,229,255}
\newcolumntype{a}{>{\columncolor{lightgray}}c}

\usepackage{etoolbox}

\makeatletter
\patchcmd{\@makecaption}
  {\parbox}
  {\advance\@tempdima-\fontdimen2} 
  {}{}
\makeatother

\usepackage{etoolbox}

\makeatletter
\patchcmd{\@makecaption}
  {\parbox}
  {\advance\@tempdima-\fontdimen2} 
  {}{}
\makeatother

\title[Support vector comparison machines]{Support vector comparison machines}
\author[David Venuto {\it et al.}]{David Venuto}
\address{Mila and McGill University,
Montreal,
         Canada.}
\email{david.venuto@mail.mcgill.ca}
\author{Toby Hocking}
\address{Northern Arizona University,
         Flagstaff,
         USA.}
\email{toby.hocking@nau.com}

\author{Lakjaree Spanurattana}
\address{Tokyo Institute of Technology,
         Tokyo,
         Japan.}
\email{bristolb7777@gmail.com}
\author[D. Venuto, T.D. Hocking, L. Spanurattana, M. Sugiyama]{Masashi Sugiyama}
\address{RIKEN Center for Advanced Intelligence Project
		and The University of Tokyo, 
		Tokyo, Japan}
\email{sugi@k.u-tokyo.ac.jp}

\begin{document}

\begin{abstract}
  In ranking problems, the goal is to learn a ranking function
  $r(\mathbf x)\in\RR$ from labeled pairs $\mathbf x,\mathbf x'$ of
  input points. In this paper, we consider the related comparison
  problem, where the label $y\in\{-1,0,1\}$ indicates which element of
  the pair is better ($y = -1$ or $1$), or if there is no significant difference ($y=0$). We
  cast the learning problem as a margin maximization, and show that it
  can be solved by converting it to a standard SVM. We use simulated
  nonlinear patterns and a real learning to rank sushi data set to
  show that our proposed SVMcompare algorithm outperforms SVMrank when
  there are equality $y=0$ pairs.  In addition, we show that SVMcompare outperforms the ELO rating system when predicting the outcome of chess matches.
\end{abstract}

\newpage

\section{Introduction}

In the supervised learning to rank problem \citep{learning-to-rank}, we are
given labeled pairs of items $\mathbf x,\mathbf x'$, where the label
$y\in\{-1,1\}$ indicates which item in the pair should be ranked
higher. The goal is to learn a ranking function $r(\mathbf x)\in\RR$
which outputs a real-valued rank for each item. In this paper we
consider a related problem in which the expanded label space
$y\in\{-1,0,1\}$ includes the $y=0$ label which indicates that there
should be no rank difference. In this context the goal is to learn a
comparison function $c(\mathbf x, \mathbf x')\in\{-1,0,1\}$.

Comparison data naturally arise from competitive two-player games in which
the space of possible outcomes includes a draw (neither player wins).
In games such as chess, draws are a common result between highly
skilled players \citep{elo_score}. To accurately predict the outcome of such games, it
is thus important to learn a model that can predict a draw. 

Comparison data also results when considering subjective human
evaluations of pairs of items. For example, if each item is a movie, a
person might say that \textit{Les Mis\'erables} is better than
\textit{Star Wars}, and \textit{The Empire Strikes Back} is as good as
\textit{Star Wars}. Another example is rating food items such as
wine, in which a person may prefer one wine to another, but not be
able to perceive a difference between two other wines. In this
context, it is important to use a model which can predict no
difference between two items.

More formally, assume
that we have a training sample of $n$ labeled pairs. For each pair
$i\in\{1,\dots,n\}$ we have input features $\mathbf x_i,\mathbf
x_i'\in\RR^p$, where $p$ is a positive integer, and a label $y_i\in\{-1,0,1\}$ that indicates which
element is better:
\begin{equation}
  \label{eq:z}
  y_i =
  \begin{cases}
    -1 &  r(\mathbf x_i)>r(\mathbf x_i')
    \text{, $\mathbf x_i$ is better than $\mathbf x'_i$},\\
    0 & r(\mathbf x_i) = r(\mathbf x_i')
    \text{, $\mathbf x_i$ is as good as $\mathbf x'_i$},\\
    1 & r(\mathbf x_i)<r(\mathbf x_i')
    \text{, $\mathbf x'_i$ is better than $\mathbf x_i$}.
  \end{cases}
\end{equation}

These data are geometrically represented in the top panel of
Figure~\ref{fig:norm-data}. Pairs with equality labels $y_i=0$ are
represented as line segments, and pairs with inequality labels
$y_i=\{-1,1\}$ are represented as arrows pointing to the item with the
higher rank.

The goal of learning is to find a comparison function $c:\RR^p \times
\RR^p \rightarrow \{-1,0,1\}$ which generalizes to a test set of data,
as measured by the zero-one loss:
\begin{equation}
  \label{eq:min_c}
  \minimize_{c} 
  \sum_{i\in\text{test}}
  I\left[ c(\mathbf x_i, \mathbf x_i')\neq y_i \right],
\end{equation}
where $I$ is the indicator function. If there are no equality $y_i=0$
pairs, then this problem is equivalent to learning to rank with a
pairwise zero-one loss function \citep{learning-to-rank}. Learning to
rank has been extensively studied, resulting in state-of-the-art
algorithms such as SVMrank \citep{ranksvm}. However, we are interested
in learning to compare with equality $y_i=0$ pairs, which to our
knowledge has only been studied by \citet{rank-with-ties} where various paired comparison models were used and trained by gradient boosting.  Paired comparison models for predicting ties included the Bradley-Terry Model \citep{bt}, Thurstone-Mosteller Model \citep{tm} and General Linear Models with a tie calling threshold.  The Bradley-Terry Model, which does not accommodate ties, can be extended to accommodate ties in paired comparisons directly \citep{bt-tie}.  Other works consider a different but somewhat related problem of ordinal regression where our outputs are in $1,\dots, K$ ordered classes to be predicted \citep{NIPS2002_2269}. In this article we propose SVMcompare, a support vector algorithm for these
data.

The notation and organization of this article is as follows. We use
bold uppercase letters for matrices such as $\mathbf X, \mathbf K$,
and bold lowercase letters for their row vectors $\mathbf x_i, \mathbf
k_i$. In Section~\ref{sec:related} we discuss links with related work
on classification and ranking, then in Section~\ref{sec:svm-compare}
we propose a new algorithm: SVMcompare. We show results on 3
illustrative simulated data sets and 2 real by learning to rank a sushi
data set and a chess dataset in Section~\ref{sec:results} and~\ref{sec:chess}. We then discuss future work in
Section~\ref{sec:conclusions}.

\section{Related work}
\label{sec:related}

First we discuss connections with several existing methods, and then
we discuss how ranking algorithms can be applied to the comparison
problem.


\begin{figure}
  \centering
  \input{figure-norm-data}
  \caption{Geometric interpretation. \textbf{Left}: input feature pairs
    $\mathbf x_i,\mathbf x_i'\in\RR^p$ are segments for $y_i=0$ and
    arrows for $y_i\in\{-1,1\}$. The level curves of the ranking
    function $r(\mathbf x)=||\mathbf x||_2^2$ are grey, and
    differences $|r(\mathbf x')-r(\mathbf x)|\leq 1$ are considered
    insignificant ($y_i=0$). \textbf{Middle}: in the enlarged feature
    space, the ranking function is linear: $r(\mathbf x)=\mathbf
    w^\intercal \Phi(\mathbf x)$. \textbf{Right}: two symmetric
    hyperplanes $\mathbf w^\intercal[\Phi(\mathbf x_i')-\Phi(\mathbf
    x_i)]\in\{-1,1\}$ are used to classify the difference vectors.}
  \label{fig:norm-data}
\end{figure}

\subsection{Comparison and ranking problems}

In ranking problems, each training example is a pair of
inputs/features $\mathbf x,\mathbf x'\in\RR^p$, and a corresponding
label/output $y_i\in\{-1,1\}$, which indicates which of the two inputs
should be ranked higher (Table~\ref{tab:related}). In the comparison
problem that we study in this paper, outputs $y\in\{-1,0,1\}$ include
the $y=0$ equality pairs or ties, which indicate that the two inputs
should be ranked equally.  The statistics literature contains many
probabilistic models for paired comparison experiments, some of which
directly model ties \citep{davidson-ties}. Such models are concerned
with accurately ranking a finite number of inputs $x\in\{1,\dots,t\}$,
so are not directly applicable to the real-valued inputs
$\mathbf x\in\RR^p$ we consider in this paper. 

\begin{table}
    \caption{\label{tab:related}Our proposed SVM for comparison is similar to previous SVM algorithms for ranking and binary classification.}
    \begin{tabular}{|c|c|c|}\hline
    Outputs/Inputs & single items $\mathbf x$ & pairs $\mathbf x,\mathbf x'$ \\ \hline
    $y\in\{-1,1\}$ &SVM  & SVMrank   	\\ \hline 
    $y\in\{-1,0,1\}$ & - & \textbf{This work: SVMcompare}\\ \hline
  \end{tabular}
\end{table}

The supervised learning to rank problem has been extensively studied
in the machine learning literature \citep{object-ranking-methods,
  learning-to-rank}, and is similar to the supervised comparison
problem we consider in this paper. There are several Bayesian models
which can be applied to learning to rank, such as TrueSkill
\citep{trueskill} and Glicko \citep{Glicko}, which are generalizations
of the Elo chess rating system \citep{elo_score}. The SVMrank algorithm was
proposed for learning to rank \citep{ranksvm}, and the large-margin
learning formulation we propose in this article is similar. The
difference is that we also consider the case where both inputs are
judged to be equally good ($y_i=0$). A boosting algorithm for this
``ranking with ties'' problem was proposed by \citet{rank-with-ties},
who observed that modeling ties is more effective when there are more
output values.

\changed{Ranking data sets are often described not in terms of labeled
  pairs of inputs $(\mathbf x_i, \mathbf x_i', y_i)$ but instead
  single inputs $\mathbf x_i$ with ordinal labels
  $y_i\in\{1,\dots,k\}$, where $k$ is the number of integral output
  values. Support Vector Ordinal Regression \citep{ordinal} has a
  large-margin learning formulation specifically designed for these
  data. Another approach is to first convert the inputs to a database
  of labeled pairs, and then learn a ranking model such as the
  SVMcompare model we propose in this paper. \citet{sv-survival}
  observed that directly using a regression model gives better
  performance than ranking models for survival data. However, in this
  paper we limit our study to models for labeled pairs of inputs, and
  we focus on answering the question, ``how can we adapt the Support
  Vector Machine to exploit the structure of the equality $y_i=0$
  pairs when they are present?''}


\subsection{SVMrank for comparing}
\label{sec:svmrank}
In this subsection we explain how to apply the existing SVMrank algorithm
to a comparison data set.  The goal of SVMrank is to learn a ranking
function $r:\RR^p \rightarrow \RR$. When $r(\mathbf x)=\mathbf
w^\intercal \mathbf x$ (where $^\top$ denotes the transpose) is linear, the primal problem for some cost
parameter $C\in\RR^+$ (where $\mathbb{R}^+$ is a set of all non-negative real numbers) is the following quadratic program (QP):
\begin{equation}
  \begin{aligned}
    \minimize_{\mathbf w, \mathbf \xi}\ \ & \frac 1 2
    \mathbf w^\intercal \mathbf w
    + C \sum_{i\in \mathcal I_1\cup \mathcal I_{-1}} \xi_i \\
    \text{subject to}\ \ &
    \forall i\in \mathcal I_1\cup \mathcal I_{-1},\ \xi_i \geq 0,\\
    & \text{and }\xi_i \geq 1-\mathbf w^\intercal(\mathbf x_i'-\mathbf
    x_i)y_i,
  \end{aligned}
  \label{eq:svmrank}
\end{equation}
where $\mathcal I_y=\{i\mid y_i=y\}$ are the sets of indices for the different
labels. Note that (\ref{eq:svmrank}) is the same as Optimization
Problem 1 (Ranking SVM), in the paper of \citet{ranksvm}. Note also
that the equality $y_i=0$ pairs are not used in the optimization
problem.

After obtaining a weight vector $\mathbf w\in\RR^p$ by solving SVMrank
(\ref{eq:svmrank}), we get a ranking function $r(\mathbf x)\in\RR$,
but we are not yet able to predict equality $y_i=0$ pairs. To do so,
we extend SVMrank by defining a threshold $\tau\in\RR^+$ and a
thresholding function $t_\tau:\RR\rightarrow\{-1,0,1\}$
\begin{equation}
  \label{eq:threshold}
  t_\tau(x) = 
  \begin{cases}
    -1 & \text{ if } x < -\tau, \\
    0 & \text{ if } |x| \leq \tau, \\
    1 & \text{ if } x > \tau. \\
  \end{cases}
\end{equation}
A comparison function $c_\tau:\RR^p\times \RR^p\rightarrow \{-1, 0,
1\}$ is defined as the thresholded difference of predicted ranks
\begin{equation}
  \label{eq:compare_general}
  c_\tau(\mathbf x, \mathbf x') = 
  t_\tau\big(
  r(\mathbf x') - r(\mathbf x)
  \big).
\end{equation}
We can then use grid search to estimate an optimal threshold $\hat
\tau$, by minimizing the zero-one loss with respect to all the
training pairs:
\begin{equation}
  \hat \tau = \argmin_{\tau}
  \sum_{i=1}^n
  I\left[ c_\tau(\mathbf x_i, \mathbf x_i') \neq y_i \right].
\end{equation}
However, there are two potential problems with the learned comparison
function $c_{\hat\tau}$. First, the equality pairs $i\in \mathcal I_0$ are not
used to learn the weight vector $\mathbf w$ in (\ref{eq:svmrank}). Second, the
threshold $\hat \tau$ is learned in a separate optimization step,
which may be suboptimal. In the next section, we propose a new
algorithm that fixes these issues by directly using all the
training pairs in a single learning problem.
\section{Support vector comparison machines}
\label{sec:svm-compare}

\begin{figure}[b!]
  \centering
  \input{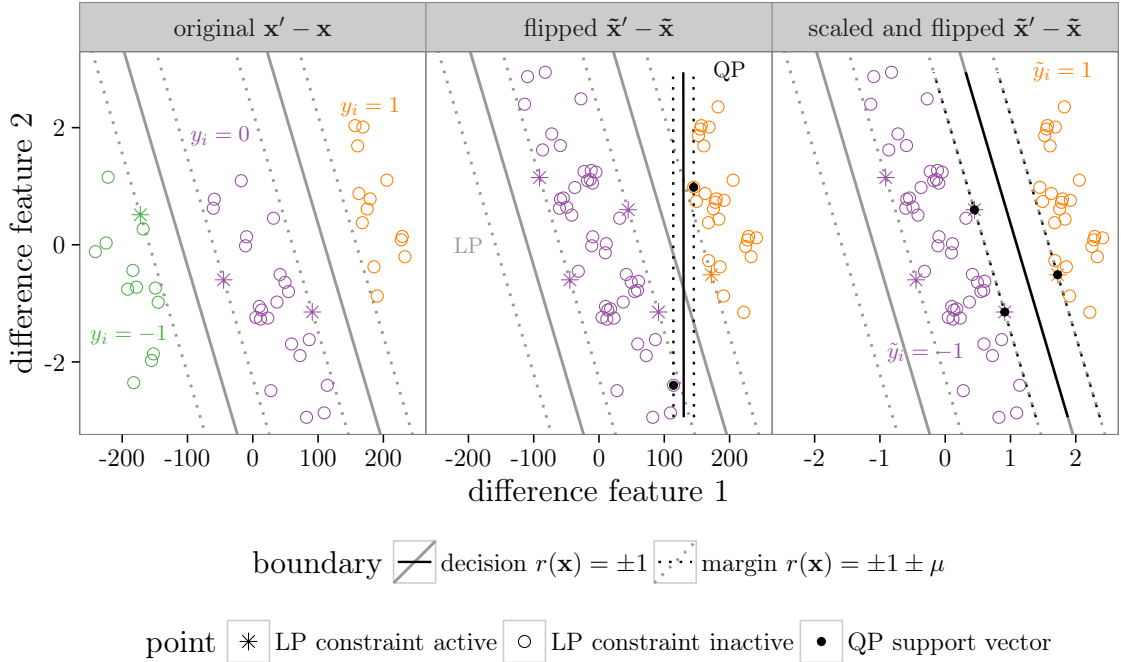}
  \caption{The separable LP and QP comparison problems. \textbf{Left}:
    the difference vectors $\mathbf x'-\mathbf x$ of the original data
    and the optimal solution to the LP
    (\ref{eq:max-margin-lp}). \textbf{Middle}: for the unscaled
    flipped data $\mathbf{\tilde x'}-\mathbf{ \tilde x}$
    (\ref{eq:tilde}), the LP is not the same as the QP
    (\ref{eq:max-margin-qp-tilde}). \textbf{Right}: in these scaled data,
    the QP is equivalent to the LP.}
  \label{fig:hard-margin}
\end{figure}

In this section we discuss new learning algorithms for comparison
problems. In all cases, we will first learn a ranking function
$r:\RR^p\rightarrow\RR$ and then a comparison function
$c_1:\RR^p\times \RR^p\rightarrow\{-1,0,1\}$ defined in
(\ref{eq:compare_general}). In other words, a small rank difference
$|r(\mathbf x')-r(\mathbf x)|\leq 1$ is considered insignificant, and
there are two decision boundaries $r(\mathbf x')-r(\mathbf
x)\in\{-1,1\}$.

\subsection{LP and QP for separable data}
\label{sec:lp-qp}

In our learning setup, the best comparison function is the one with
maximum margin. We will define the margin in two different ways, which
correspond to the linear program (LP) and quadratic program (QP)
discussed below. To illustrate the differences between these
max-margin comparison problems, in this subsection we assume that the
training data are linearly separable. Later in
Section~\ref{sec:kernelized-qp}, we propose an algorithm for learning
a nonlinear function from non linearly-separable data.

In the following linear program, we learn a linear ranking function
$r(\mathbf x)=\mathbf w^\intercal \mathbf x$ that maximizes the margin
$\mu$, defined in terms of ranking function values. The margin $\mu$
is the smallest rank difference between a decision boundary $r(\mathbf
x)\in\{-1,1\}$ and a difference vector $r(\mathbf x_i'-\mathbf
x_i)$. The max margin LP is
\begin{eqnarray}
  \label{eq:max-margin-lp}
  \maximize_{\mu\in\RR^+, \mathbf w\in\RR^p}\ &&\hskip -0.5cm \mu \\
  \nonumber
  \text{subject to}\ && \hskip -0.5cm \mu \leq
  1-|\mathbf w^\intercal (\mathbf x_i' - \mathbf x_i)|,\ 
  \forall\  i\in \mathcal I_0,\\
  \nonumber
  &&\hskip -0.5cm
  \mu \leq -1 +  \mathbf w^\intercal(\mathbf x_i'-\mathbf x_i)y_i,
  \ \forall\ i\in \mathcal I_1\cup \mathcal I_{-1}.
\end{eqnarray}
The optimal decision boundaries $r(\mathbf x)\in\{-1,1\}$ and margin
boundaries $r(\mathbf x)\in\{-1\pm \mu, 1 \pm \mu\}$ are drawn 
in Figure~\ref{fig:hard-margin}. Note that finding a feasible point
for this LP is a test of linear separability. If there are no feasible
points then the data are not linearly separable.

Another way to formulate the comparison problem is by first performing
a change of variables, and then solving a binary SVM QP. The idea is
to maximize the margin between significant differences
$y_i\in\{-1,1\}$ and equality pairs $y_i=0$. Let $\mathbf X_y,\mathbf
X_y'$ be the $|\mathcal I_y|\times p$ matrices formed by all the pairs
$i\in \mathcal I_y$. We define a new ``flipped'' data set with
$m=|\mathcal I_1|+|\mathcal I_{-1}|+2|\mathcal I_0|$ pairs suitable
for training a binary SVM:
\begin{equation}
\label{eq:tilde}
\mathbf{  \tilde X} = \left[
    \begin{array}{c}
      \mathbf X_1 \\
      \mathbf X_{-1}'\\
      \mathbf X_0\\
      \mathbf X_0'
    \end{array}
  \right],\ 
  \mathbf{\tilde X'} = \left[
    \begin{array}{c}
      \mathbf X_1' \\
      \mathbf X_{-1}\\
      \mathbf X_0'\\
      \mathbf X_0
    \end{array}
  \right],\ 
  \mathbf{\tilde y} = \left[
    \begin{array}{c}
      \mathbf 1_{|\mathcal I_1|} \\
      \mathbf 1_{|\mathcal I_{-1}|}\\
      \mathbf{-1}_{|\mathcal I_0|}\\
      \mathbf{-1}_{|\mathcal I_0|}
    \end{array}
  \right],
\end{equation}
where $\mathbf 1_n$ is an $n$-vector of ones, $\mathbf{\tilde
  X},\mathbf{\tilde X'}\in\RR^{m\times p}$ and $\mathbf{\tilde
  y}\in\{-1,1\}^m$. Note that $\tilde y_i=-1$ implies no significant
difference between $\mathbf{\tilde x}_i$ and $\mathbf{\tilde x}_i'$,
and $\tilde y_i=1$ implies that $\mathbf{\tilde x}_i'$ is better than
$\mathbf{\tilde x}_i$. We then learn an affine function $f(\mathbf
x)=\beta+\mathbf u^\intercal \mathbf x$ using a binary SVM QP:
\begin{eqnarray}
  \label{eq:max-margin-qp-tilde}
  \minimize_{\mathbf u\in\RR^p, \beta\in\RR}\ &&\hskip -0.5cm
  \mathbf u^\intercal \mathbf u  \\
\nonumber    \text{subject to}\ &&\hskip -0.5cm 
    \tilde y_i (\beta + 
    \mathbf u^\intercal( \mathbf{\tilde x}_i'-\mathbf{\tilde x}_i) ) \geq 1,
    \ \forall i\in\{1,\dots,m\}.
\end{eqnarray}
This SVM QP learns a separator $f(\mathbf x)=0$ between significant
difference pairs $\tilde y_i=1$ and insignificant difference pairs
$\tilde y_i=-1$ (middle and right panels of
Figure~\ref{fig:hard-margin}). However, we want a comparison function
that predicts $c(\mathbf x,\mathbf x')\in\{-1,0,1\}$. So we use the
next lemma to construct a ranking function $r(\mathbf x)= \mathbf{\hat
  w}^\intercal \mathbf x$ that is feasible for the original max margin
comparison LP (\ref{eq:max-margin-lp}), and can be used with the
comparison function~$c_1$, defined in (\ref{eq:compare_general}).
\begin{lemma}
  Let $\mathbf u\in\RR^p,\beta\in\RR$ be a solution of
  (\ref{eq:max-margin-qp-tilde}). Then $\hat \mu = -1/\beta$ and
  $\mathbf{\hat w} = -\mathbf u/\beta$ are feasible for
  (\ref{eq:max-margin-lp}).
  \label{lemma:feasible}
\end{lemma}
\begin{proof}
  Begin by assuming that we want to find a ranking function $r(\mathbf
  x)=\mathbf{\hat w}^\intercal \mathbf x = \gamma \mathbf u^\intercal
  \mathbf x$, where $\gamma\in\RR$ is a scaling constant. Then
  consider that for all $\mathbf x$ on the decision boundary, we have
  \begin{equation}
    \label{eq:dec-boundary-rank}
    r(\mathbf x) = \mathbf{\hat w}^\intercal \mathbf x = 1\text{ and } 
    f(\mathbf x) = \mathbf u^\intercal \mathbf x + \beta = 0.
  \end{equation}
  Taken together, it is clear that $\gamma=-1/\beta$ and thus
  $\mathbf{\hat w} = -\mathbf u/\beta$. Consider for all $\mathbf x$ on the
  margin we have
  \begin{equation}
    \label{eq:margin-rank}
    r(\mathbf x) = \mathbf{\hat w}^\intercal \mathbf x = 1+\hat\mu\text{ and } 
    f(\mathbf x) = \mathbf u^\intercal \mathbf x + \beta= 1.
  \end{equation}
  Taken together, these imply $\hat \mu=-1/\beta$. Now, by definition
  of the flipped data (\ref{eq:tilde}), we can re-write the max margin
  QP (\ref{eq:max-margin-qp-tilde}) as
\begin{eqnarray}
  \label{eq:max-margin-qp}
    \minimize_{\mathbf u\in\RR^p, \beta\in\RR}\ &&
    \hskip -0.5cm \mathbf u^\intercal \mathbf u  \\
    \text{subject to}\ &&\hskip -0.5cm
    \nonumber \beta + |\mathbf u^\intercal (\mathbf x_i'-\mathbf x_i)| \leq -1,\
    \forall\  i\in \mathcal I_0,\\
    &&\hskip -0.5cm
\nonumber \beta + \mathbf u^\intercal(\mathbf x_i'-\mathbf x_i)y_i \geq 1,
\ \forall\ i \in \mathcal I_1\cup \mathcal I_{-1}.
\end{eqnarray}
By re-writing the constraints of (\ref{eq:max-margin-qp}) in terms of
$\hat \mu$ and $\mathbf{\hat w}$, we recover the same constraints as
the max margin comparison LP (\ref{eq:max-margin-lp}). Thus $\hat \mu,
\mathbf{\hat w}$ are feasible for (\ref{eq:max-margin-lp}). (QED)
\end{proof}

One may also wonder: are $\hat \mu,\mathbf{\hat w}$ optimal for the
max margin comparison LP? In general, the answer is no, and we give
one counterexample in the middle panel of
Figure~\ref{fig:hard-margin}. This is because the LP defines the
margin in terms of ranking function values $r(\mathbf x)=\mathbf
w^\intercal \mathbf x$, but the QP defines the margin in terms of the
size of the normal vector $||\mathbf u||$, which depends on the scale
of the inputs $\mathbf x,\mathbf x'$. However, when the input
variables are scaled in a pre-processing step, we have observed that
the solutions to the LP and QP are equivalent (right panel of
Figure~\ref{fig:hard-margin}).

Lemma~\ref{lemma:feasible} establishes the fact that one can learn a
ranking function $r$ and a corresponding comparison function $c_1$, defined in
(\ref{eq:compare_general}), by solving either the LP
(\ref{eq:max-margin-lp}) or the QP (\ref{eq:max-margin-qp}). To make
corresponding learning problems for non linearly-separable data as defined by the linear separability test in this section, one can add
slack variables to either the QP or the LP. In the next subsection, we
pursue only the QP, since it leads to a dual problem with a sparse
solution that can be solved by any standard SVM solver such as libsvm
\citep{libsvm}.

\subsection{Kernelized QP for non-separable data}
\label{sec:kernelized-qp}
In this subsection, we assume the data are not linearly separable, and want to
learn a nonlinear ranking function. We define a positive definite
kernel $\kappa:\RR^p\times \RR^p\rightarrow\RR$, which implicitly
defines an enlarged set of features $\Phi(\mathbf x)$ (middle panel of
Figure~\ref{fig:norm-data}). As in (\ref{eq:max-margin-qp-tilde}), we
learn a function $f(\mathbf x)=\beta + \mathbf u^\intercal
\Phi(\mathbf x)$ which is affine in the feature space. Let $\mathbf
\alpha,\mathbf \alpha'\in\RR^m$ be coefficients such that $\mathbf
u=\sum_{i=1}^m
\alpha_i \Phi(\mathbf{\tilde x}_i) + 
\alpha_i' \Phi(\mathbf{\tilde x}_i')$, and so we have
 $f(\mathbf x) =\beta + \sum_{i=1}^m 
 (\alpha_i \kappa(\mathbf{\tilde x}_i, \mathbf x) +
 \alpha_i' \kappa(\mathbf{\tilde x}_i', \mathbf x))$. 
 We then use Lemma~\ref{lemma:feasible} to
define the ranking function
\begin{equation}
  \label{eq:kernelized_r}
  r(\mathbf x)= \frac{\mathbf u^\intercal \Phi(\mathbf x)}{-\beta} = 
  \sum_{i=1}^m
  \frac{
    \alpha_i \kappa(\mathbf{\tilde x}_i, \mathbf x) +
    \alpha_i'  \kappa(\mathbf{\tilde x}_i', \mathbf x)}
{-\beta}.
\end{equation}

\begin{algorithm}[b!]
   \caption{SVMcompare}
   \label{alg:SVMcompare}
\begin{algorithmic}
  \STATE {\bfseries Input:} cost $C\in\RR^+$, kernel
  $\kappa:\RR^p\times \RR^p \rightarrow \RR$, features $\mathbf
  X,\mathbf X'\in\RR^{n \times p}$, labels $\mathbf y\in\{-1,0,1\}^n$.

  \STATE \makebox[0.5cm]{$\mathbf{\tilde X}$} $\gets [$
  \makebox[1cm]{$\mathbf X_1^\intercal$}
  \makebox[1cm]{$\mathbf X_{-1}'^\intercal$}
  \makebox[1cm]{$\mathbf X_0^\intercal$}
  \makebox[1cm]{$\mathbf X_0'^\intercal$}
  $]^\intercal$.

  \STATE \makebox[0.5cm]{$\mathbf{\tilde X}'$} $\gets [$
  \makebox[1cm]{$\mathbf X_1'^\intercal$}
  \makebox[1cm]{$\mathbf X_{-1}^\intercal$}
  \makebox[1cm]{$\mathbf X_0'^\intercal$}
  \makebox[1cm]{$\mathbf X_0^\intercal$}
  $]^\intercal$.

  \STATE \makebox[0.5cm]{$\mathbf{\tilde y}$} $\gets [$
  \makebox[1cm]{$\mathbf 1_{|\mathcal I_1|}^\intercal$}
  \makebox[1cm]{$\mathbf 1_{|\mathcal I_{-1}|}^\intercal$}
  \makebox[1cm]{$-\mathbf 1_{|\mathcal I_0|}^\intercal$}
  \makebox[1cm]{$-\mathbf 1_{|\mathcal I_0|}^\intercal$}
  $]^\intercal$.

  \STATE $\mathbf K \gets \proc{KernelMatrix}(
  \mathbf{\tilde X}, \mathbf{\tilde X'}, \kappa)$.

  \STATE $\mathbf M \gets [ -\mathbf I_m\ \mathbf I_m ]^\intercal$.

  \STATE $\mathbf{\tilde K} \gets \mathbf M^\intercal \mathbf  K \mathbf M$.

  \STATE $\mathbf v,\beta \gets \proc{SVMdual}(
  \mathbf{\tilde K}, \mathbf{\tilde y}, C)$.

  \STATE $\sv \gets\{i: v_i>0\}$.
  
  \STATE {\bfseries Output:} Support vectors $\mathbf{\tilde
    X}_{\sv },\mathbf{\tilde X}_{\sv }'$, labels
  $\mathbf{\tilde y}_{\sv }$, bias~$\beta$, dual variables $\mathbf v$.

   \end{algorithmic}
\end{algorithm}

Let $\mathbf K=[
\mathbf k_1\cdots \mathbf k_m
\ \mathbf k_1'\cdots \mathbf k_m']\in\RR^{2m\times 2m}$ be the
kernel matrix, where for all pairs $i\in\{1, \dots, m\}$, the kernel
vectors $\mathbf k_i,\mathbf k_i'\in\RR^{2m}$ are defined as
\begin{equation}
  \mathbf k_i = \left[
    \begin{array}{c}
      \kappa(\mathbf{\tilde x}_1, \mathbf{\tilde x}_i)\\
      \vdots\\
      \kappa(\mathbf{\tilde x}_m, \mathbf{\tilde x}_i)\\
      \kappa(\mathbf{\tilde x}_1', \mathbf{\tilde x}_i)\\
      \vdots\\
      \kappa(\mathbf{\tilde x}_m', \mathbf{\tilde x}_i)
    \end{array}
  \right],\ 
  \mathbf k_i' = \left[
    \begin{array}{c}
      \kappa(\mathbf{\tilde x}_1, \mathbf{\tilde x}_i')\\
      \vdots\\
      \kappa(\mathbf{\tilde x}_m, \mathbf{\tilde x}_i')\\
      \kappa(\mathbf{\tilde x}_1', \mathbf{\tilde x}_i')\\
      \vdots\\
      \kappa(\mathbf{\tilde x}_m', \mathbf{\tilde x}_i')
    \end{array}
  \right].
\end{equation}
Letting $\mathbf a=[\alpha^\intercal\
\alpha'^\intercal]^\intercal\in\RR^{2m}$, the norm of the affine
function $f$ in the feature space is $\mathbf u^\intercal \mathbf u =
\mathbf a^\intercal \mathbf K \mathbf a$, and we can write the primal soft-margin
comparison QP for some $C\in\RR^+$ as
\begin{eqnarray}
  \minimize_{\mathbf a\in\RR^{2m},\mathbf \xi\in\RR^m,\beta\in\RR}\ \ &&\hskip -0.5cm 
  \frac 1 2 \mathbf a^\intercal \mathbf K \mathbf a + C\sum_{i=1}^m \xi_i \\
  \text{subject to}\ \ &&\hskip -0.5cm \nonumber
  \text{for all $i\in\{1,\dots,m\}$, }
  \xi_i \geq 0,\\
  &&\hskip -0.5cm \nonumber \text{and }
  \xi_i \geq 1-\tilde y_i(\beta + \mathbf a^\intercal (\mathbf k_i'-\mathbf k_i)).
\end{eqnarray}
Let $\mathbf z, \mathbf v\in\RR^m$ be the dual variables, and
$\mathbf Y=\Diag(\mathbf{\tilde y})$ be the diagonal matrix of $m$
labels. Then the Lagrangian can be written as
\begin{equation}
  \label{eq:lagrangian}
  \mathcal L = \frac 1 2 \mathbf a^\intercal \mathbf K \mathbf a + 
  C\mathbf \xi^\intercal\mathbf  1_{m}\\
  -\mathbf z^\intercal \mathbf \xi
  + \mathbf v^\intercal(\mathbf 1_m - \mathbf{\tilde y}\beta
  - \mathbf Y \mathbf M^\intercal \mathbf K\mathbf  a - \xi),
\end{equation}
where $\mathbf M=[-\mathbf I_m \, \mathbf
I_m]^\intercal\in\{-1,0,1\}^{2m\times m}$ and $\mathbf I_m$ is the
identity matrix. Solving $\nabla_{\mathbf a} \mathcal L=0$ results in
the following stationary condition:
\begin{equation}
  \label{eq:stationarity}
  \mathbf a = \mathbf M \mathbf Y \mathbf v.
\end{equation}
The rest of the derivation of the dual comparison problem is the same
as for the standard binary SVM. The resulting dual QP is
\begin{equation}
  \begin{aligned}
    \label{eq:svm-dual}
    \minimize_{\mathbf v\in\RR^m}\ \ &
    \frac 1 2 \mathbf v^\intercal \mathbf Y \mathbf M^\intercal 
    \mathbf K \mathbf M \mathbf Y \mathbf v - \mathbf v^\intercal \mathbf 1_m\\
    \text{subject to}\ \ &
    \sum_{i=1}^m v_i \tilde y_i = 0,\\
    & \text{for all $i\in\{1,\dots,m\}$, } 0\leq v_i\leq C,
  \end{aligned}
\end{equation}
which is equivalent to the dual problem of a standard binary SVM with
kernel $\mathbf{\tilde K} = \mathbf M^\intercal \mathbf K \mathbf
M\in\RR^{m\times m}$ and labels $\mathbf{\tilde y}\in\{-1,1\}^m$.

So we can solve the dual comparison problem (\ref{eq:svm-dual}) using
any efficient SVM solver, such as libsvm \citep{libsvm}. We used the R
interface in the \texttt{kernlab} package \citep{kernlab}, and our
code is available in the \texttt{rankSVMcompare} package on Github.

After obtaining optimal dual variables $\mathbf v\in\RR^m$ as the solution of
(\ref{eq:svm-dual}), the SVM solver also gives us the optimal bias
$\beta$ by analyzing the complementary slackness conditions.
The learned ranking function can be quickly evaluated since the
optimal $\mathbf v$ is sparse. Let $\sv =\{i: v_i > 0\}$ be the indices
of the support vectors. Since we need only $2|\sv |$ kernel
evaluations, the ranking function (\ref{eq:kernelized_r}) becomes
\begin{equation}
  \label{eq:r_sv}
  r(\mathbf x)= 
  \sum_{i\in \sv }
  \tilde y_i v_i\left[ 
    \kappa(\mathbf{\tilde x}_i, \mathbf x)
    - \kappa(\mathbf{\tilde x}_i', \mathbf x)
  \right]/\beta.
\end{equation}
Note that for all $i\in\{1,\dots,m\}$, the optimal primal variables
$\alpha_i=-\tilde y_i v_i$ and $\alpha_i'=\tilde y_i v_i$ are
recovered using the stationary condition (\ref{eq:stationarity}). The
learned comparison function $c_1$, as defined in (\ref{eq:compare_general}), remains the same.

The training procedure is summarized as
Algorithm~\ref{alg:SVMcompare}, SVMcompare.
There are two sub-routines: \proc{KernelMatrix} computes
the $2m\times 2m$ kernel matrix, and \proc{SVMdual} solves the SVM
dual QP (\ref{eq:svm-dual}). There are two hyper-parameters to tune:
the cost $C$ and the kernel $\kappa$. As with standard SVM for binary
classification, these parameters can be tuned by minimizing the
prediction error on a held-out validation set.

\section{Comparison to SVMrank in sushi and simulated data sets}
\label{sec:results}



\begin{table}
    \caption{\label{tab:models}
    Summary of how the different algorithms 
    use the input pairs to learn the ranking 
    function $r$. Equality $y_i=0$ pairs are shown as 
    segments and inequality $y_i \in \{-1,1\}$ pairs 
    are shown as $\rightarrow$  arrows. For example, 
    the rank2 algorithm converts each input equality pair
    to two opposite-facing inequality pairs.}
\centering
  \begin{tabular}{r|cc|cc|}
Input:&    \multicolumn{2}{c|}{equality pairs}
&    \multicolumn{2}{c|}{inequality pairs}\\
    & $|\mathcal I_0|$ 
    & --- 
    & $|\mathcal I_1|+|\mathcal I_{-1}|$ 
    & $\rightarrow$
    \\
    \hline
    rank 
    & 0 
    & 
    & $|\mathcal I_1|+|\mathcal I_{-1}|$ 
    & $\rightarrow$ 
    \\
    \hline
    rank2 
    & $2|\mathcal I_0|$ 
    & $\leftarrow \rightarrow$
    & $2(|\mathcal I_1|+|\mathcal I_{-1}|)$ 
    & $\rightarrow \rightarrow$
    \\
    \hline
    compare 
    & $2|\mathcal I_0|$ 
    & --- --- 
    & $|\mathcal I_1|+|\mathcal I_{-1}|$ 
    & $\rightarrow$\\
    \hline
  \end{tabular}
\end{table}

The goal of learning to compare is to accurately predict a test set of
labeled pairs (\ref{eq:min_c}), which includes equality $y_i=0$
pairs. We test the SVMcompare algorithm alongside two baseline models
that use SVMrank \citep{ranksvm}. We chose SVMrank as a baseline
because of its similar large-margin learning formulation, to
demonstrate the importance of directly modeling the equality $y_i=0$
pairs. SVMrank does not directly model the equality $y_i=0$ pairs, so
we expect that the proposed SVMcompare algorithm makes better
predictions when these data are present. The differences between the
algorithms are summarized in Table~\ref{tab:models}:

\begin{description}
\item[rank] is described in Section~\ref{sec:svmrank}: first we use
  $|\mathcal I_1|+|\mathcal I_{-1}|$ inequality pairs to learn
  SVMrank, then we use all $n$ pairs to learn a threshold $\hat \tau$
  for when to predict $c(\mathbf x,\mathbf x')=0$.
\item[rank2] is another variant of SVMrank that treats each input pair
  as 2 inequality pairs. Since SVMrank can only use inequality pairs,
  we transform each equality pair $(\mathbf x_i,\mathbf x_i',0)$ into two
  opposite-facing inequality pairs $(\mathbf x_i',\mathbf x_i,1)$ and
  $(\mathbf x_i,\mathbf x_i',1)$. 
  To ensure equal weight for all input pairs in the
  cost function, we also duplicate each inequality pair, resulting in
  $2n$ pairs used to train SVMrank.
\item[compare] is the SVMcompare model proposed in this paper, which
  uses $m=n+|\mathcal I_0|$ input pairs.
\end{description}

For each experiment, there are train, validation, and test sets each
drawn from the same data set of examples. We fit a $10\times 10$ grid of
models to the training set (cost parameter $C=10^{-3},\dots,10^3$,
Gaussian kernel width $2^{-7},\dots,2^4$), and select the model using
the validation set. We use two evaluation metrics to judge the
performance of the models: zero-one loss and area under the ROC curve
(AUC).

Note that the ROC curves are calculated by first evaluating the
learned ranking function $r(\mathbf x)$ at each test point $\mathbf
x$, and then varying the threshold $\tau$ of the comparison function
$c_\tau$, as defined in (\ref{eq:compare_general}). For $\tau=0$ we have 100\% false
positive rate and for $\tau=\infty$ we have 100\% false negative rate
(Table~\ref{tab:evaluation}).

\begin{table}
  \caption{We use area under the ROC curve to evaluate predictions
    $\hat y$ given the true label $y$. False positives (FP) occur 
    when predicting a significant difference $\hat y\in\{-1,1\}$ 
    when there is none ($y=0$). False Negatives (FN) occur when
    a labeled difference $y\in\{-1,1\}$ is incorrectly predicted.}
  \centering
  \begin{tabular}{|p{1cm}|p{1cm}|p{1cm}|p{1cm}|}\hline
    \rowcolor{lightgray}
    $\hat{y}$/ $y$
    &\textbf{-1}&\textbf{0}&\textbf{1}\\ \hline
    \textbf{-1}&0  & FP & FN   	\\ \hline 
    \textbf{0} &FN& 0& FN\\ \hline
    \textbf{1} & FN & FP &0	\\ \hline
  \end{tabular}
  \label{tab:evaluation}
\end{table}

\subsection{Simulation: squared norms in 2D}
\label{sec:simulations}

We used a simulation to visualize the learned nonlinear ranking
functions in a $p=2$ dimensional feature space. We generated pairs
$\mathbf x_i,\mathbf x_i'\in[-3,3]^2$ and noisy labels
$y_i=t_1[r(\mathbf x'_i)-r(\mathbf x_i)+\epsilon_i]$, where $t_1$ is
the threshold function, (\ref{eq:threshold}), $r$ is the latent ranking
function, $\epsilon_i\sim N(\mu\ = 0, \sigma)$ is noise, and $\sigma=1/4$ is
the standard deviation. We picked train, validation, and test sets,
each with the same number of pairs $n$ and the same proportion $\rho$
of equality pairs. We selected the model with minimum zero-one loss on
the validation set.

In Figure~\ref{fig:norm-level-curves} we defined the true ranking
function $r(\mathbf x)=||\mathbf x||_1^2$, then picked $n=100$ pairs,
with $\rho=1/2$ equality and inequality pairs. We show the training
set and the level curves of the ranking functions learned by the
SVMrank and SVMcompare models. It is clear that the true ranking
function $r$ is not accurately recovered by the rank model, since it
does not use the equality $y_i=0$ pairs. In contrast, the compare and
rank2 methods which exploit the equality $y_i=0$ pairs are able to
recover a ranking function that is closer to the true $r$.

We also used the simulations to demonstrate that our model can achieve
lower test error than the baseline SVMrank model, by learning from the
equality $y_i=0$ pairs. In Figure~\ref{fig:simulation-samples} we
fixed the proportion of equality pairs $\rho=1/2$, varied the number
of training pairs $n\in\{50,\dots, 800\}$, and tested three simulated
ranking functions $r(\mathbf x)=||\mathbf x||^2_j$ for
$j\in\{1,2,\infty\}$. In general, the test error of all models
decreases as training set size $n$ increases. The model with the
highest test error is the rank model, which does not use the equality
$y_i=0$ pairs. The next worst model is the rank2 model, which converts
the equality $y_i=0$ pairs to inequality pairs and then trains
SVMrank. The best model is the proposed SVMcompare model, which
achieves test error as good as the true ranking function in the case
of $r(\mathbf x)=||\mathbf x||^2_2$.

\begin{figure}[b!]
  \centering
  \input{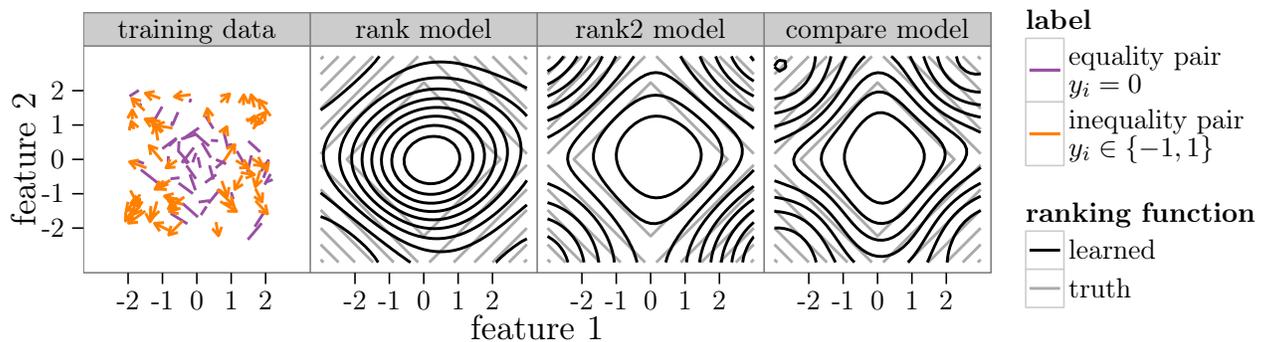}
  \caption{Application to a simulated pattern $r(\mathbf x)=||\mathbf
    x||_1^2$ where $\mathbf x\in\RR^2$. \textbf{Left}: the training
    data are $n=100$ pairs, half equality (segments indicate two
    points of equal rank), and half inequality (arrows point to the
    higher rank). \textbf{Others}: level curves of the learned ranking
    functions. The rank model does not directly model the equality
    pairs, so the rank2 and compare models recover the true pattern
    better.}
  \label{fig:norm-level-curves}
\end{figure}

\begin{figure}[b!]
  \input{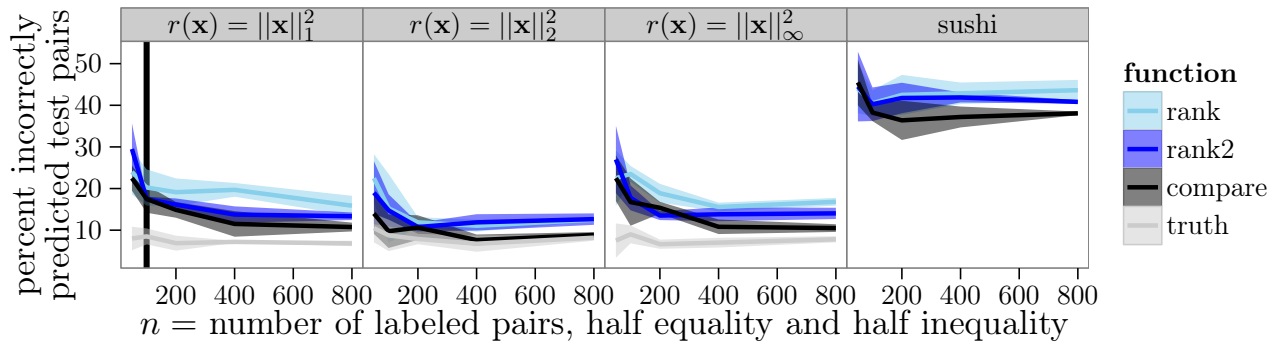}
  \caption{Test error for 3 different simulated patterns $r(\mathbf
    x)$ where $\mathbf x\in\RR^2$ and one real
    sushi data set where $\mathbf x\in\RR^{14}$. We randomly generated data sets
    with $\rho=1/2$ equality and 1/2 inequality pairs, then plotted test
    error as a function of data set size $n$ (a vertical line
    shows the data set which was used in
    Figure~\ref{fig:norm-level-curves}). Lines show mean and shaded
    bands show standard deviation over 4 test sets.}
  \label{fig:simulation-samples}
\end{figure}

\begin{figure}[!ht]
  \input{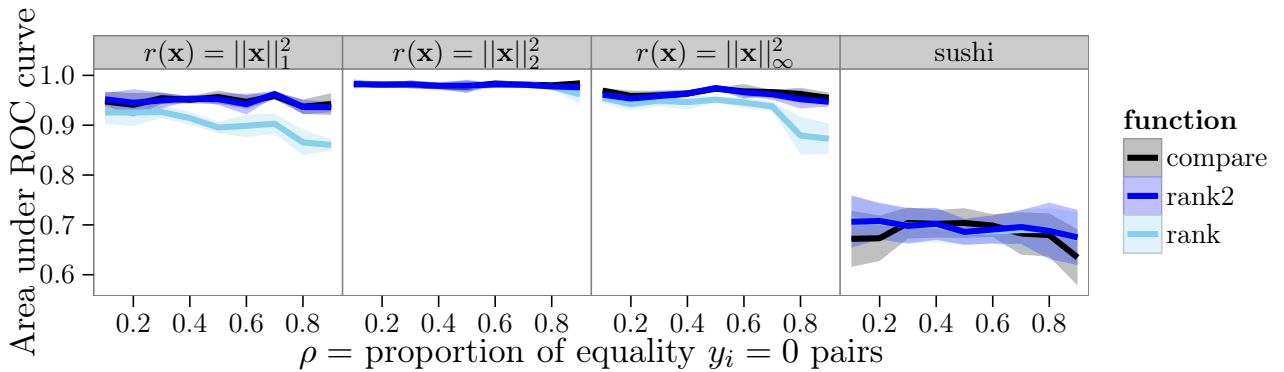}
  \caption{Area under the ROC curve (AUC) for 3 different simulated
    patterns $r(\mathbf x)$ where $\mathbf x\in\RR^2$ and one real
    sushi data set where $\mathbf x\in\RR^{14}$. For each data set we
    picked $n=400$ pairs, varying the proportion $\rho$ of equality
    pairs. We plot mean and standard deviation of AUC over 4 test
    sets.}
  \label{fig:auc}
\end{figure}

In Figure~\ref{fig:auc} we fixed the number of training pairs $n=400$
and varied the proportion $\rho$ of equality pairs for the three
simulated squared norm ranking functions $r$. We select the model with
maximum area under the validation set ROC curve, then use test set AUC
to evaluate the learned models. All methods perform close to the
optimal true ranking function when $r(\mathbf x)=||\mathbf
x||^2_2$. For the other patterns, it is clear that all the methods
perform similarly when there are mostly inequality pairs ($\rho=0.1$),
since SVMrank was designed for this type of training data. In
contrast, when there are mostly equality pairs ($\rho=0.9$), the
compare and rank2 methods clearly outperform the rank method, which
ignores the equality pairs. \changed{It is also clear that the rank2 and
compare methods perform similarly in terms of test AUC.}

\changed{Overall from the simulations, it is clear that when the data contain
equality pairs, it is advantageous to use a model such as the proposed
SVMcompare method which learns from them directly as a part of the
optimization problem.}

\subsection{Learning to rank sushi data}

We downloaded the sushi data set of \citet{object-ranking-methods}
from kamishima \\(http://www.kamishima.net/sushi). We used the
\textit{sushi3b.5000.10.score} from kamishima, which consist of 100 different
sushis rated by 5000 different people. Each person rated 10 sushis on
a 5 point scale, which we convert to 5 preference pairs, for a total
of 17,832 equality $y_i=0$ and 7,168 inequality $y_i\in\{-1,1\}$
pairs. For each pair $i$ we have features $\mathbf x_i,\mathbf
x_i'\in\RR^{14}$ consisting of 7 features of the sushi and 7 features
of the person. Sushi features are style, major, minor, oily, eating
frequency, price, and selling frequency. Person features are gender, age,
time, birthplace and current home (we converted Japanese prefecture
codes to latitude/longitude coordinates).
As in the simulations of Section~\ref{sec:simulations}, we picked
train, validation, and test sets, each with the same number of pairs
$n$ and the same proportion $\rho$ of equality pairs. We fit a grid of
models to the training set, select the model with
minimal zero-one loss on the validation set, and then use the test set
to estimate the generalization ability of the selected model.

In Figure~\ref{fig:simulation-samples} we fixed the proportion of
equality pairs $\rho=1/2$, varied the number of training pairs
$n\in\{50,\dots, 800\}$, and calculated test error. The relative
performance of the algorithms is the same as in the simulations: rank
has the highest test error, rank2 does better, and the proposed
SVMcompare algorithm has the lowest test error.

In Figure~\ref{fig:auc} we fixed the number of training pairs $n=400$,
varied the proportion $\rho$ of equality pairs, and calculated test
AUC. \changed{Like in the $r(\mathbf x)=||\mathbf x||_2^2$ simulation,
  test AUC is about the same for each of the models
  considered. Perhaps this is because they are all able to learn a
  nearly optimal ranking function for this problem.}

\changed{Overall from the sushi data, it is clear that the proposed SVMcompare
model performs better than the SVMrank methods in terms of test error,
and it performs as well as the SVMrank methods in terms of test AUC.
}

\section{SVMcompare predict outcomes of chess games more accurately than ELO}
\label{sec:chess}

In this section, we show that our SVMcompare algorithm can be used for
highly accurate prediction of the outcome of chess matches. Chess is a
game between 2 players that results in a win, loss or a draw, with draws being very common between highly ranked players in international
tournaments. We wished to predict outcomes of tournament
chess matches by learning a comparison function using features based
on player statistics.  The main statistic to quantify player rankings
in FIDE (Fédération Internationale des Échecs) competitions is the ELO
score.  The ELO rating system is a method for calculating relative
skill levels of players in competitor versus competitor games was
intially proposed by \citet{elo_score}.
The Glicko rating system 
provides a more complex alternative \citep{Glicko}.

\subsection{Data source and processing}

We downloaded the chess match dataset from Chessmetrics \\(http://www.chessmetrics.com/cm/), containing
1.8 million games played over the 11-year period from 1999--2009 by
54205 chess players.  For each of the years 1999--2006, we consider the
first four months (Jan--Apr) as a train set, and the last eight months
as a test set (May--Dec). 
We
removed all matches containing a player who had less than 10 matches
against other players in the train set, to prevent our data set from
containing players with very little information.  We also removed all
matches that contained a player's first match from the train set as we
would have no information about this player.  Before pre-processing, 30.9\% of matches were a draw and 69.1\% of matches resulted in a win or loss.  After pre-processing, the median percentage of draws and win-loss results was 44.7\% and 55.3\% respectively over each of the 8 datasets.
For each match $i$, we computed features
$\mathbf x_i,\mathbf x_i'\in\RR^{16}$ consisting of ELO
scores, Glicko scores, if the player had the initial move, the
percentage of instances where a player either lost to a lower ranked
player, or won against a higher ranked player, the average score
difference of opponents, win/loss/draw/games played raw values and
percentages in addition to various other statistics. ELO scores were
initially set at 1200 for all players and FIDE rules were applied to
score calculations.  ELO and Glicko scores were updated after every
match using the PlayerRatings R package  \citep{play-raitings}.  

\subsection{Cross-validation experiment setup}

For hyper-parameter selection, we used the first 3 months of each
12-month period.  This was done for computational speed reasons.  We performed cross validation splits of the first
$\{0.50,0.75,0.80,0.85\}$ matches in the set as our training set and
the remainder as the validation set.  We fit a grid of models for each
linear, polynomial and Gaussian kernel to the training set to select a
model with the maximum AUC on the validation set.  For all kernels,
the grid of cost parameters was
$C\in\{10^{-20},10^{-18}, \dots,10^0\}$. The Gaussian kernel width was
$10^{-1},10^{0},10^1$ and the polynomial kernel degree was
$1,2,3,4$.  

We then used the selected hyper-parameters to train a
model using the first four months (Jan--Apr) of each year, and we used
the learned model to predict on the test set for that year
(May--Dec). We then computed the test AUC for each of the eight years. Since there are 3 labels $y\in\{-1,0,1\}$ corresponding to $\{$win, draw, loss$\}$ respectively, the trivial AUC is non-standard (not 0.5). In our trivial calculation, we obtain a FPR and TPR of 0 by predicting every observation as a negative class ($y=0$).  By predicting every observation as the most common positive class, we have obtained the maximum TPR without use any features.  Our predictions for the trivial model are therefore calculated with:
\begin{equation}c_{\tau}(x)=
\begin{cases} 1 &
\text{ if } |
\mathcal I_{-1}|<|\mathcal I_1| \\ -1
& \text{ otherwise.} 

\end{cases}
\end{equation}
  We then obtain the FPR given all observations are predicted to be the most common positive prediction.  Since we have 2 distinct positive classes, it is highly unlikely that our TPR and FPR will be 1.0 through this method.

\subsection{Results and discussion}

As shown in Figure~\ref{fig:chess}, the linear and polynomial SVM kernels have
higher test AUC than the ELO and Glicko scoring systems.
Additionally, linear and polynomial SVM models trained using only ELO
and Glicko features perform worse than models using all features.  This suggests that the additional features computed in our model were
relevant in learning models that performed better than ELO scores.  Models using a combination of ELO and Glicko features obtain an AUC that appears to be the median of the Glicko and ELO AUC values.  The model appears to be learning a combination of these features and therefore obtains an AUC intermediate of ELO and Glicko scores only.  The Gaussian kernel also preforms with a lower AUC with respect to the linear and polynomial kernels.  The fact that the linear kernel is more accurate than the Gaussian kernel suggests that the pattern is linear in the features we have computed.

Overall, our analysis of the chess match data suggests that the
proposed SVMcompare model performs with a higher AUC than the existing
state-of-the-art ELO and Glicko results.

\begin{figure}
	\centering
	\input{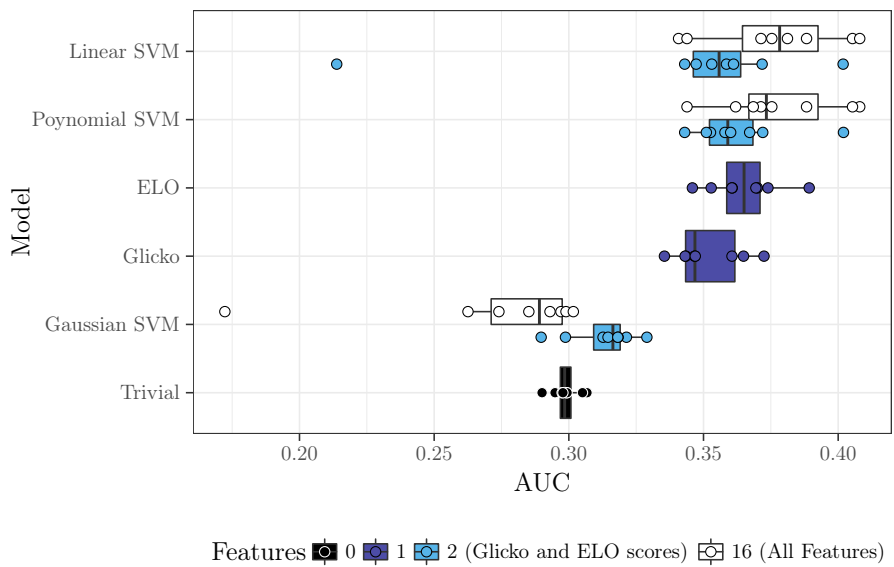}	
	\caption{Test AUC for each model used after training on the first 4 months of match data in the 8 different 12-month periods.  All SVM model AUCs are shown in addition to AUC of the ELO, Glicko scores and trivial model.  The plots in black show the trivial AUC calculated by predicting the most common positive label.  The plots in dark blue are AUC values obtained from using Glicko or ELO scores only.  The plots in light blue show the AUC distribution from models using only ELO and Glicko features and plots in white are AUC values from models using all computed features.}
	\label{fig:chess}
\end{figure}

\section{Conclusions and future work}
\label{sec:conclusions}

We discussed the learning to compare problem, which has not yet been
extensively studied in the machine learning literature. In
Section~\ref{sec:lp-qp}, we proposed two different formulations for
max-margin comparison, and proved their relationship in
Lemma~\ref{lemma:feasible}. It justifies our proposed SVMcompare
algorithm, which uses a binary SVM dual QP solver to learn a nonlinear
comparison function. \changed{In future work it will be interesting to explore
the learning capability of a kernelized version of the LP
(\ref{eq:max-margin-lp}) with slack variables added to the objective
function.}

Our experimental results on simulated and real data clearly showed the
importance of directly modeling the equality pairs, when they are
present. We showed in Figure~\ref{fig:auc} that when there are few
equality pairs, as is the usual setup in learning to rank problems,
the baseline SVMrank algorithm performs as well as our proposed
SVMcompare algorithm. However, when there are many equality pairs, it
is clearly advantageous to use a model such as SVMcompare which
directly learns from the equality pairs.

Out results also indicate that the proposed model performs better for
predicting outcomes of chess matches than the current FIDE ELO player
ranking system, and that incorporating additional features into out
model gives an increase in accuracy.

For future work, it will be interesting to see if the same results are
observed in learning to rank data from search engines. For scaling to
these very large data sets, we would like to try algorithms based on
smooth discriminative loss functions, such as stochastic gradient
descent with a logistic loss.

\section*{Acknowledgements}
DV was funded by an NSERC operating grant and the DeepMind Graduate Award. TDH was funded by KAKENHI 23120004, SS by a
MEXT scholarship, and MS by KAKENHI 17H01760. 
Thanks to Simon Lacoste-Julien and Hang Li for helpful discussions.

\section*{Reproducible Research Statement}
The informations and scripts used to produce the results in this paper are at :

https://github.com/tdhock/compare-paper


\bibliographystyle{agsm}
\bibliography{refs}

\begin{thebibliography}{16}
\providecommand{\natexlab}[1]{#1}
\providecommand{\url}[1]{\texttt{#1}}
\expandafter\ifx\csname urlstyle\endcsname\relax
  \providecommand{\doi}[1]{doi: #1}\else
  \providecommand{\doi}{doi: \begingroup \urlstyle{rm}\Url}\fi

\bibitem[Bradley and Terry(1952)]{bt}
Ralph~A. Bradley and Milton~E. Terry.
\newblock The rank analysis of incomplete block designs: I. the method of
  paired comparisons.
\newblock \emph{Biometrika}, 39, 1952.

\bibitem[Chang and Lin(2011)]{libsvm}
Chih-Chung Chang and Chih-Jen Lin.
\newblock {LIBSVM}: A library for support vector machines.
\newblock \emph{ACM Transactions on Intelligent Systems and Technology},
  2:\penalty0 27:1--27:27, 2011.

\bibitem[Chu and Keerthi(2005)]{ordinal}
Wei Chu and S~Sathiya Keerthi.
\newblock New approaches to support vector ordinal regression.
\newblock In \emph{Proceedings of the 22nd international conference on Machine
  learning}, pages 145--152. ACM, 2005.

\bibitem[Davidson(1970{\natexlab{a}})]{bt-tie}
Roger~R. Davidson.
\newblock On extending the bradley-terry model to accommodate ties in paired
  comparison experiments.
\newblock \emph{Journal of the American Statistical Association}, 65\penalty0
  (329):\penalty0 317--328, 1970{\natexlab{a}}.

\bibitem[Davidson(1970{\natexlab{b}})]{davidson-ties}
Roger~R Davidson.
\newblock {On Extending the Bradley-Terry Model to Accommodate Ties in Paired
  Comparison Experiments}.
\newblock \emph{Journal of the American Statistical Association}, 65\penalty0
  (329):\penalty0 317--328, March 1970{\natexlab{b}}.

\bibitem[Elo(1978)]{elo_score}
Arpad Elo.
\newblock The rating of chess players, past and present.
\newblock \emph{Acro Publishing}, 1978.

\bibitem[Glickman(1999)]{Glicko}
Mark~E Glickman.
\newblock Parameter estimation in large dynamic paired comparison experiments.
\newblock \emph{Appl. Statist.}, 48\penalty0 (3):\penalty0 377--394, 1999.

\bibitem[Herbrich et~al.(2006)Herbrich, Minka, and Graepel]{trueskill}
Ralf Herbrich, Tom Minka, and Thore Graepel.
\newblock Trueskill™: A {Bayesian} skill rating system.
\newblock In \emph{Advances in Neural Information Processing Systems}, pages
  569--576, 2006.

\bibitem[Joachims(2002)]{ranksvm}
Thorsten Joachims.
\newblock Optimizing search engines using clickthrough data.
\newblock In \emph{KDD}, 2002.

\bibitem[Kamishima et~al.(2010)Kamishima, Kazawa, and
  Akaho]{object-ranking-methods}
Toshihiro Kamishima, Hideto Kazawa, and Shotaro Akaho.
\newblock A survey and empirical comparison of object ranking methods.
\newblock \emph{Preference Learning}, pages 181--201, 2010.

\bibitem[Karatzoglou et~al.(2004)Karatzoglou, Smola, Hornik, and
  Zeileis]{kernlab}
Alexandros Karatzoglou, Alex Smola, Kurt Hornik, and Achim Zeileis.
\newblock kernlab -- an {S4} package for kernel methods in {R}.
\newblock \emph{Journal of Statistical Software}, 11\penalty0 (9):\penalty0
  1--20, 2004.

\bibitem[Li(2011)]{learning-to-rank}
Hang Li.
\newblock A short introduction to learning to rank.
\newblock \emph{IEICE Transactions on Information and Systems}, E94-D\penalty0
  (10), 2011.

\bibitem[Stephenson and Sonas(2016)]{play-raitings}
Alec Stephenson and Jeff Sonas.
\newblock {PlayerRatings: Dynamic Updating Methods for Player Ratings
  Estimation (R package version 1.0-1)}.
\newblock \emph{CRAN}, 2016.

\bibitem[Thurstone(1927)]{tm}
Louis~Leon Thurstone.
\newblock A law of comparative judgement.
\newblock \emph{Psychological Review}, 34\penalty0 (4):\penalty0 273--286,
  1927.

\bibitem[Van~Belle et~al.(2011)Van~Belle, Pelckmans, Van~Huffel, and
  Suykens]{sv-survival}
Vanya Van~Belle, Kristiaan Pelckmans, Sabine Van~Huffel, and Johan~AK Suykens.
\newblock Support vector methods for survival analysis: a comparison between
  ranking and regression approaches.
\newblock \emph{Artificial Intelligence in Medicine}, 53\penalty0 (2):\penalty0
  107--118, 2011.

\bibitem[Zhou et~al.(2008)Zhou, Xue, Zha, and Yu]{rank-with-ties}
Ke~Zhou, Gui-Rong Xue, Hongyuan Zha, and Yong Yu.
\newblock Learning to rank with ties.
\newblock In \emph{Proc. ACM SIGIR 31}, SIGIR '08, pages 275--282, New York,
  NY, 2008.

\end{thebibliography}


\begin{thebibliography}{xx}

\harvarditem{Bradley \harvardand\ Terry}{1952}{bt}
Bradley, R.~A. \harvardand\ Terry, M.~E.  \harvardyearleft
  1952\harvardyearright , `The rank analysis of incomplete block designs: I.
  the method of paired comparisons', {\em Biometrika} {\bf 39}.

\harvarditem{Chang \harvardand\ Lin}{2011}{libsvm}
Chang, C.-C. \harvardand\ Lin, C.-J.  \harvardyearleft 2011\harvardyearright ,
  `{LIBSVM}: A library for support vector machines', {\em ACM Transactions on
  Intelligent Systems and Technology} {\bf 2},~27:1--27:27.

\harvarditem{Chu \harvardand\ Keerthi}{2005}{ordinal}
Chu, W. \harvardand\ Keerthi, S.~S.  \harvardyearleft 2005\harvardyearright ,
  New approaches to support vector ordinal regression, {\em in} `Proceedings of
  the 22nd international conference on Machine learning', ACM, pp.~145--152.

\harvarditem{Davidson}{1970{\em a}}{bt-tie}
Davidson, R.~R.  \harvardyearleft 1970{\em a}\harvardyearright , `On extending
  the bradley-terry model to accommodate ties in paired comparison
  experiments', {\em Journal of the American Statistical Association} {\bf
  65}(329),~317--328.

\harvarditem{Davidson}{1970{\em b}}{davidson-ties}
Davidson, R.~R.  \harvardyearleft 1970{\em b}\harvardyearright , `{On Extending
  the Bradley-Terry Model to Accommodate Ties in Paired Comparison
  Experiments}', {\em Journal of the American Statistical Association} {\bf
  65}(329),~317--328.

\harvarditem{Elo}{1978}{elo_score}
Elo, A.  \harvardyearleft 1978\harvardyearright , `The rating of chess players,
  past and present', {\em Acro Publishing} .

\harvarditem{Glickman}{1999}{Glicko}
Glickman, M.~E.  \harvardyearleft 1999\harvardyearright , `Parameter estimation
  in large dynamic paired comparison experiments', {\em Appl. Statist.} {\bf
  48}(3),~377--394.

\harvarditem[Herbrich et~al.]{Herbrich, Minka \harvardand\
  Graepel}{2006}{trueskill}
Herbrich, R., Minka, T. \harvardand\ Graepel, T.  \harvardyearleft
  2006\harvardyearright , Trueskill™: A {Bayesian} skill rating system, {\em
  in} `Advances in Neural Information Processing Systems', pp.~569--576.

\harvarditem{Joachims}{2002}{ranksvm}
Joachims, T.  \harvardyearleft 2002\harvardyearright , Optimizing search
  engines using clickthrough data, {\em in} `KDD'.

\harvarditem[Kamishima et~al.]{Kamishima, Kazawa \harvardand\
  Akaho}{2010}{object-ranking-methods}
Kamishima, T., Kazawa, H. \harvardand\ Akaho, S.  \harvardyearleft
  2010\harvardyearright , `A survey and empirical comparison of object ranking
  methods', {\em Preference Learning} pp.~181--201.

\harvarditem[Karatzoglou et~al.]{Karatzoglou, Smola, Hornik \harvardand\
  Zeileis}{2004}{kernlab}
Karatzoglou, A., Smola, A., Hornik, K. \harvardand\ Zeileis, A.
  \harvardyearleft 2004\harvardyearright , `kernlab -- an {S4} package for
  kernel methods in {R}', {\em Journal of Statistical Software} {\bf
  11}(9),~1--20.

\harvarditem{Li}{2011}{learning-to-rank}
Li, H.  \harvardyearleft 2011\harvardyearright , `A short introduction to
  learning to rank', {\em IEICE Transactions on Information and Systems} {\bf
  E94-D}(10).

\harvarditem{Shashua \harvardand\ Levin}{2003}{NIPS2002_2269}
Shashua, A. \harvardand\ Levin, A.  \harvardyearleft 2003\harvardyearright ,
  Ranking with large margin principle: Two approaches, {\em in} S.~Becker,
  S.~Thrun \harvardand\ K.~Obermayer, eds, `Advances in Neural Information
  Processing Systems 15', pp.~961--968.

\harvarditem{Stephenson \harvardand\ Sonas}{2016}{play-raitings}
Stephenson, A. \harvardand\ Sonas, J.  \harvardyearleft 2016\harvardyearright ,
  `{PlayerRatings: Dynamic Updating Methods for Player Ratings Estimation (R
  package version 1.0-1)}', {\em CRAN} .

\harvarditem{Thurstone}{1927}{tm}
Thurstone, L.~L.  \harvardyearleft 1927\harvardyearright , `A law of
  comparative judgement', {\em Psychological Review} {\bf 34}(4),~273--286.

\harvarditem[Van~Belle et~al.]{Van~Belle, Pelckmans, Van~Huffel \harvardand\
  Suykens}{2011}{sv-survival}
Van~Belle, V., Pelckmans, K., Van~Huffel, S. \harvardand\ Suykens, J.~A.
  \harvardyearleft 2011\harvardyearright , `Support vector methods for survival
  analysis: a comparison between ranking and regression approaches', {\em
  Artificial Intelligence in Medicine} {\bf 53}(2),~107--118.

\harvarditem[Zhou et~al.]{Zhou, Xue, Zha \harvardand\ Yu}{2008}{rank-with-ties}
Zhou, K., Xue, G.-R., Zha, H. \harvardand\ Yu, Y.  \harvardyearleft
  2008\harvardyearright , Learning to rank with ties, {\em in} `Proc. ACM SIGIR
  31', SIGIR '08, New York, NY, pp.~275--282.

\end{thebibliography}

\end{document}